%% file: KLModelBasedHal.tex
\documentclass{article}
 \usepackage{amssymb}
 \usepackage{graphicx}
 \usepackage{subfigure} 

\usepackage[active]{srcltx}

 \usepackage{amsfonts,amsmath}
\usepackage{color} 
 \usepackage{algorithm}
 \usepackage{algorithmicx,algpseudocode}

\usepackage{cases}
\newcommand{\Ptrans}{P}
\newcommand{\mR}{r}
\newcommand{\e}{\ensuremath{\mathbf{e}}}

\newcommand{\norm}[2]{{\left\Vert #1 \right\Vert}_{#2}}

\newcommand{\argmax}{\operatornamewithlimits{argmax}}
\newcommand{\argmin}{\operatornamewithlimits{argmin}}
\newcommand{\reg}{\text{Regret}}

\def\1{\mathbf{1}}
\def\P{\mathbb{P}}
\def\E{\mathbb{E}}

\def\M{\mathbf{M}}
\def\Me{\mathcal{M}}

\def\Xset{\ensuremath{\mathcal{X}}}

\def\Rset{\ensuremath{\mathbb{R}}}
\def\S{\ensuremath{\mathbb{S}}}

\def\Aset{\ensuremath{\mathcal{A}}}
\def\eqsp{\;}
\def\eqsp{\;}

\newtheorem{theorem}{Theorem}
\newtheorem{proposition}{Proposition}
\newtheorem{proof}{Proof}

\newcommand{\CPE}[3][]
{\ifthenelse{\equal{#1}{}}{\operatorname{E}\left[\left. #2 \, \right| #3 \right]}{\operatorname{E}^{#1}\left[\left. #2 \, \right | #3 \right]}}                
\newcommand{\CP}[3][]
{\ifthenelse{\equal{#1}{}}{\mathbb{P}\left[\left. #2 \, \right| #3 \right]}{\mathbb{P}^{#1}\left[\left. #2 \, \right | #3 \right]}}

\title{Optimism in Reinforcement Learning\\ and Kullback-Leibler Divergence}

\author{Sarah Filippi, Olivier Capp\'{e} and Aur\'{e}lien Garivier\\
LTCI, TELECOM ParisTech and CNRS\\
46 rue Barrault, 75013 Paris, France\\
Email: (filippi, cappe, garivier)@telecom-paristech.fr\thanks{This work has been partially supported by Orange Labs under contract n\textsuperscript{o}289365.}
}

\begin{document}

\maketitle

\begin{abstract}
  We consider model-based reinforcement learning in finite Markov Decision
  Processes (MDPs), focussing on so-called optimistic strategies. In MDPs,
  optimism can be implemented by carrying out extended value iterations under a
  constraint of consistency with the estimated model transition
  probabilities. The UCRL2 algorithm by Auer, Jaksch and Ortner (2009), which
  follows this strategy, has recently been shown to guarantee near-optimal
  regret bounds.
  In this paper, we strongly argue in favor of using the Kullback-Leibler (KL)
  divergence for this purpose. By studying the linear maximization problem
  under KL constraints, we provide an efficient algorithm, termed KL-UCRL, for
  solving KL-optimistic extended value iteration.
  Using recent deviation bounds on the KL divergence, we prove that KL-UCRL
  provides the same guarantees as UCRL2 in terms of regret. However, numerical
  experiments on classical benchmarks show a significantly improved behavior,
  particularly when the MDP has reduced connectivity. To support this
  observation, we provide elements of comparison between the two
  algorithms based on geometric considerations.
\end{abstract}

\section{Introduction}

In reinforcement learning, an agent interacts with an unknown environment,
aiming to maximize its long-term payoff~\cite{Sutton:Barto:98}. This interaction is
modelled by a Markov Decision Process (MDP) and it is assumed that the
agent does not know the parameters of the process and needs to learn
directly from observations. The agent thus faces a fundamental trade-off between
gathering experimental data about the consequences of the actions (exploration)
and acting consistently with past experience in order to maximize the rewards
(exploitation).\\
\hspace*{0.35cm}We consider in this article a MDP with finite state and action spaces for which
we propose a \textit{model-based} reinforcement learning algorithm, i.e., an
algorithm that maintains running estimates of the model parameters
(transitions probabilities and expected rewards) \cite{EvenDar:al:06,Kearns:Singh:02,Strehl:Littman:08,Tewari:Bartlett:08}.
A well-known approach to balance exploration and exploitation, followed for example by the well-know algorithm R-MAX \cite{Brafman:Tennenholtz:03}, is the so-called \textit{optimism in the face of uncertainty} principle. It was first proposed in the
multi-armed bandit context by \cite{Lai:Robbins:85}, and has been extended since then to several frameworks: instead of acting
optimally according to the estimated model, the agent follows the optimal
policy for a surrogate model, named \textit{optimistic model}, which is close enough to the former but leads to a higher long-term reward.
The performance of such an algorithm can be analyzed in terms of \textit{regret}, which consists 
in comparing the rewards collected by the algorithm with the rewards obtained when following an optimal policy. 
The study of the asymptotic regret due to \cite{Lai:Robbins:85} in the multi-armed context has been extended to MDPs by \cite{Burnetas:Katehakis:97}, proving that an optimistic algorithm can achieve logarithmic regret.
The subsequent works \cite{Auer:Ortner:07,Jaksch:al:10,Bartlett:Tewari:09}
introduced algorithms that guarantee non-asymptotic logarithmic regret in a large class of MDPs. In these latter works, the optimistic model is computed using the $L^1$ (or total variation) norm as a measure of proximity between the estimated and optimistic transition probabilities.

In addition to logarithmic regret bounds, the UCRL2 algorithm of \cite{Jaksch:al:10} is also attractive due to the simplicity of each $L^1$ extended value iteration step. In this case, optimism simply results in adding a bonus to the most promising transition (i.e., the transition that leads to the state with current highest value) while removing the corresponding probability mass from less promising transitions. This process is both elementary and easily interpretable, which is desirable in some applications.

However, the $L^1$ extended value iteration leads to undesirable pitfalls, which may compromise the practical performance of the algorithm. First, the optimistic model is not continuous with respect to the estimated parameters -- small changes in the estimates may result in very different optimistic models. More importantly, the $L^1$ optimistic model can become incompatible with the observations by assigning a probability of zero to a transition that has actually been observed. Moreover, in MDPs with reduced connectivity, $L^1$ optimism results in a persistent bonus for all transitions heading towards the most valuable state, even when significant evidence has been accumulated that these transitions are impossible.

In this paper, we propose an improved optimistic algorithm, called KL-UCRL, that avoids these pitfalls altogether. The key is the use of the Kullback-Leibler (KL) pseudo-distance instead of the $L^1$ metric, as in \cite{Burnetas:Katehakis:97}. Indeed, the smoothness of the KL metric largely alleviates the first issue. The second issue is completely avoided thanks to the strong relationship between the geometry of the probability simplex induced by the KL pseudo-metric and the theory of large deviations. For the third issue, we show that the KL-optimistic model results from a trade-off between the relative value of the most promising state and the statistical evidence accumulated so far regarding its reachability.

We provide an efficient procedure, based on one-dimensional line searches, to solve the linear maximization problem under KL constraints. As a consequence, the numerical complexity of the KL-UCRL algorithm is comparable to that of UCRL2.
Building on the analysis of \cite{Jaksch:al:10,Bartlett:Tewari:09,Auer:al:09}, we also obtain logarithmic regret bounds for the KL-UCRL algorithm. The proof of this result is based on novel concentration inequalities for the KL-divergence, which have interesting properties when compared with those traditionally used for the $L^1$ norm.
Although the obtained regret bounds are comparable to earlier results in term
of rate and dependence in the number of states and actions, we observed in
practice significant performance improvements. This observation is illustrated
using benchmark examples (the \emph{RiverSwim} and \emph{SixArms} environments of
\cite{Strehl:Littman:08}) and through a thorough discussion of the geometric
properties of KL neighborhoods.

The paper is organized as follows. The model and a brief survey of the value iteration algorithm in undiscounted MDPs are presented in
Section~\ref{sec:modele}. Section~\ref{sec:algo} and~\ref{sec:regret} are devoted, respectively, to the description
and the analysis of the KL-UCRL algorithm. Section \ref{sec:simu}
contains numerical experiments and Section~\ref{sec:KLversusL1} concludes the paper by discussing
the advantages of using KL rather than $L^1$ confidence neighborhoods.

\section{Markov Decision Process}\label{sec:modele}

Consider a Markov decision process (MDP) $\M=(\Xset,\Aset,\Ptrans, \mR)$ with finite state space  $\Xset$, and action space $\Aset$. 
Let $X_t\in\Xset$ and $A_t\in\Aset$ denote respectively the state of the system and the action chosen by the agent at time $t$.
The probability to jump from state $X_t$ to state $X_{t+1}$ is denoted by $\Ptrans(X_{t+1};X_t,A_t)$. 
Besides, the agent receives at time $t$ a random reward $R_t\in[0,1]$ with mean $\mR(X_t,A_t)$. 
The aim of the agent is to choose the sequence of actions so as to maximize the cumulated reward.
His choices are summarized in a \textit{stationary policy} $\pi:\Xset\rightarrow\Aset$.

In this paper, we consider \textit{communicating} MDPs, i.e., MDPs such that for any pair of states $x,x'$, there exists policies under which $x'$ can be reached from $x$ with positive probability. For those MDPs, it is known that the \textit{average reward} following a stationary policy $\pi$, denoted by $\rho^{\pi}(\M)$ and defined as 

\[
  \rho^\pi(\M) = \lim\limits_{n\to \infty}\frac{1}{n} \E_{\M, \pi}\left(\sum_{t=0}^{n}R_t\right) \eqsp,
\]
is state-independent~\cite{Puterman:94}.  
Let $\pi^*(\M):\Xset\rightarrow\Aset$ and $\rho^*(\M)$ denote respectively the optimal policy and the optimal average reward: $
\rho^*(\M)=\sup_\pi  \rho^\pi(\M)= \rho^{\pi^*(\M)}(\M)\eqsp.
$
The notations $\rho^*(\M)$ and $\pi^*(\M)$ are meant to highlight the fact that
both the optimal average reward and the optimal policy depend on the model $\M$. 
The optimal average reward satisfies the so-called \textit{Bellman optimality equation}: for all $x\in\Xset$,
\begin{multline*}
h^*(\M,x) + \rho^*(\M) = \\
  \max_{a\in\Aset}\left( \mR(x,a)+\sum_{x'\in\Xset}\Ptrans(x';x,a)h^*(\M,x')\right) \eqsp ,
\end{multline*}
where the $|\Xset|$-dimensional vector $h^*(\M)$ is called a \emph{bias} vector. Note that it is only defined up to an additive constant. 
For a fixed MDP $\M$, the optimal policy $\pi^*(\M)$ can be derived by solving the optimality equation and by defining, for all $x\in\Xset$,
$$\pi^*(\M,x)\in\argmax_{a\in\Aset}\left(\mR(x,a)+\sum_{x'\in\Xset}\Ptrans(x';x,a) h^*(\M,x)\right).$$
In practice, the optimal average reward and the optimal policy may be computed, for instance, using the value iteration algorithm~\cite{Puterman:94}.


\section{The KL-UCRL algorithm} \label{sec:algo}
In this paper, we focus on the reinforcement learning problem in which the agent does not know the model $\M$ beforehand, i.e. the transition probabilities and the distribution of the rewards are unknown. More specifically, we consider model-based reinforcement learning algorithms which estimate the model through observations and act accordingly. Denote by $\hat{\Ptrans}_t(x';x,a)$ the estimate at time $t$ of the transition probability from state $x$ to state $x'$ conditionally to the action $a$, and, by $\hat{r}_t(x,a)$ the mean reward received in state $x$ when action $a$ has been chosen. We have:
\begin{align}
\hat{\Ptrans}_t(x';x,a) &= \frac{N_t(x,a,x')}{\max(N_t(x,a),1)}\nonumber\\
\hat{r}_t(x,a) &= \frac{\sum_{k=0}^{t-1}R_k\1_{\{X_k=x,A_k=a\}}}{\max(N_t(x,a),1)}\eqsp,\label{eq:est} 
\end{align}
where $N_t(x,a,x')=\sum_{k=0}^{t-1}\1_{\{X_k=x,A_k=a,X_{k+1}=x'\}}$ is the number of visits, up to time $t$, to the state $x$ followed by a visit to $x'$ when the action $a$ has been chosen, and similarly, $N_t(x,a)=\sum_{k=0}^{t-1}\1_{\{X_k=x,A_k=a\}}$.
The optimal policy in the estimated model $\widehat\M_t=(\Xset,\Aset,\hat\Ptrans_t,\hat{r}_t)$ may be misleading due to estimation errors: pure exploitation policies are commonly known to fail with positive probability. To avoid this problem, \textit{optimistic model-based approaches} consider a set $\Me_t$ of potential MDPs including $\widehat\M_t$ and choose the MDP from this set that leads to the largest average reward. In the following, the set $\Me_t$ is defined as follows:
 \begin{align*}
\Me_t&=\{\M=(\Xset,\Aset,\Ptrans, \mR) : \forall x\in\Xset,\forall a\in\Aset,\\
&\quad\quad\quad\quad|\hat\mR_t(x,a)-\mR(x,a)|\leq \epsilon_R(x,a,t) \\
&\quad\quad\quad\text{and } d(\hat\Ptrans_t(.;x,a),\Ptrans(.;x,a))\leq \epsilon_P(x,a,t)\}\eqsp,
 \end{align*}
where $d$ measures the difference between the transition probabilities.
The \textit{radius of the neighborhoods} $\epsilon_R(x,a,t)$ and $\epsilon_P(x,a,t)$ around, respectively, the estimated reward $\hat\mR_t(x,a)$ and the estimated transition probabilities $\hat\Ptrans_t(.;x,a)$, decrease with $N_t(x,a)$.

In contrast to UCRL2, which uses the $L^1$-distance for $d$, we propose to rely on the Kullback-Leibler divergence, as in the seminal article~\cite{Burnetas:Katehakis:97}; however, contrary to the approach of \cite{Burnetas:Katehakis:97}, no prior knowledge on the state structure of the MDP is needed. Recall that the Kullback-Leibler divergence is defined for all $n$-dimensional probability vectors $p$ and $q$ by $KL(p,q)=\sum_{i=1}^n p_i\log\frac{p_i}{q_i}$ (with the convention that $0\log 0=0$).
In the sequel, we will show that this choice dramatically alters the behavior of the algorithm and leads to significantly better performance, while causing a limited increase of complexity; in Section~\ref{sec:KLversusL1}, the advantages of using a KL-divergence instead of the $L^1$-norm are illustrated and argumented.

\subsection{The KL-UCRL algorithm}
The KL-UCRL algorithm, described below, is a variant of the efficient model-based algorithm UCRL2, introduced by \cite{Auer:al:09} and extended to more general MDPs by \cite{Bartlett:Tewari:09}. The key step of the algorithm, the search for the optimistic model (Step 8), is detailed below as Algorithm \ref{algo:MaxKL}.

\begin{algorithm}[hbtp]
\begin{algorithmic}[1]
\caption{KL-UCRL}
\label{algo:klucrl}
\State Initialization: $j=0$, $t_0 = 0$; $\forall a\in\Aset,\forall x\in\Xset,n_0(x,a)=0$, $N_0(x,a)=0$; initial policy $\pi_0$.
\For{all $t\geq 1$}
\State Observe $X_t$
\If{$n_j(X_t,\pi_j(X_t))\geq\max(N_{t_j}(X_t,\pi_j(X_t)),1)$}
\State \textit{Begin a new episode:} $j=j+1$, $t_j=t$,
\State Reinitialize: $\forall a\in\Aset,\forall x\in\Xset\eqsp,\eqsp n_j(x,a)=0$ 
\State Estimate $\hat\Ptrans_t$ and $\hat{r}_t$ according to \eqref{eq:est}
\State Find the optimistic model $\M_j\in\Me_t$ and the related policy $\pi_{j}$ solving equation~\eqref{eq:extOpt} and using Algorithm~\ref{algo:MaxKL}
\EndIf
\State Choose action $A_t = \pi_j(X_t)$
\State Receive reward $R_t$
\State Update the count within the current episode: $$n_j(X_t,A_t)=n_j(X_t,A_t)+1$$
\State Update the global count: $$N_t(X_t,A_t)=N_{t-1}(X_t,A_t)+1$$
\EndFor

\end{algorithmic}
\end{algorithm}

The KL-UCRL algorithm proceeds in episodes. Let $t_j$ be the starting time of episode $j$; the length of the $j$-th episode depends on the number of visits $N_{t_j}(x,a)$ to each state-action pair $(x,a)$ before $t_j$ compared to the number of visits $n_j(x,a)$ to the same pair during the $j$-th episode. More precisely, an episode ends as soon as $n_j(x,a)\geq N_{t_j}(x,a)$ for some state-action pair $(x,a)$. The policy $\pi_j$, followed during the $j$-th episode, is an optimal policy for
the optimistic MDP $\M_j=(\Xset,\Aset,\Ptrans_j,\mR_j)\in\Me_{t_j}$,
 which is computed by solving the \textit{extended optimality equations}: for all $x\in\Xset$
\begin{equation}
h^*(x) + \rho^* = \max_{\Ptrans,\mR}\max_{a\in\Aset}\left( \mR(x,a)+\sum_{x'\in\Xset}\Ptrans(x';x,a)h^*(x')\right)\label{eq:extOpt}
\end{equation}
where the maximum is taken over all $P, r$ such that
\begin{align*}
& \forall x, \forall a,\quad KL(\hat\Ptrans_{t_j}(.;x,a),\Ptrans(.;x,a))\leq \frac{C_P}{N_{t_j}(x,a)} \eqsp , \\ 
&\forall x, \forall a,\quad |\hat\mR_{t_j}(x,a)-\mR(x,a)|\leq \frac{C_R}{\sqrt{N_{t_j}(x,a)}} \eqsp,
\end{align*}
where $C_P$ and $C_R$ are constants which control the size of the confidence balls. 
The transition matrix $\Ptrans_j$ and the mean reward $\mR_j$ of the optimistic MDP $\M_j$ maximize those equations.
The \textit{extended value iteration} algorithm may be used to approximately solve the fixed point equation~\eqref{eq:extOpt} \cite{Puterman:94,Auer:al:09}.

\subsection{Maximization of a linear function on a KL-ball}\label{sec:maxKL}

At each step of the extended value iteration algorithm, the maximization problem~\eqref{eq:extOpt} has to be solved. For every state $x$ and action $a$, the maximization of $r(x,a)$ under the constraint that $|\hat\mR_{t_j}(x,a)-\mR(x,a)|\leq C_R/\sqrt{N_{t_j}(x,a)}$ is obviously solved taking $\mR(x,a)=\hat{r}_{t_j}(x,a)+C_R/\sqrt{N_{t_j}(x,a)}$, so that the main difficulty lies in maximizing the dot product between the probability vector $q=\Ptrans(.;x,a)$ and the 
\textit{value vector} $V=h^*$ over a KL-ball around the fixed probability vector $p=\hat\Ptrans_{t_j}(.;x,a)$:
\begin{equation}
\max_{q\in\S^{|\Xset|}}V'q\quad\text{s.t.}\quad KL(p,q)\leq\epsilon\eqsp,\label{eq:max_kl}
\end{equation}
where $V'$ denotes the transpose of $V$ and $\S^n$ the set of $n$-dimensional probability vectors. The radius of the neighborhood $\epsilon=C_P/N_{t_j}(x,a)$ controls the size of the confidence ball. This convex maximization problem is studied in Appendix~\ref{ap:solmax}, leading to the efficient algorithm presented below. 
Detailed analysis of the Lagrangian of~\eqref{eq:max_kl} shows that the solution of the 
maximization problem 
essentially relies on finding roots of the function $f$ (that depends on the parameter $V$), defined as follows: for all $\nu\geq \max_{i\in\bar{Z}}V_i$, with $\bar{Z}=\{i:p_i>0\}$,
\begin{equation}f(\nu) = \sum_{i\in\bar{Z}}p_i\log(\nu-V_i)  + \log\left(\sum_{i\in\bar{Z}}\frac{p_i}{\nu-V_i}\right)\eqsp.\label{eq:def_f}\end{equation}
In the special case where the most promising state $i_M$ has never been reached from the current state-action pair (i.e. $p_{i_M}=0$), the algorithm makes a trade-off between the relative value of the most promising state $V_{i_M}$ and the statistical evidence accumulated so far regarding its reachability.
\begin{algorithm}[hbtp]
\begin{algorithmic}[1]
\caption{Function MaxKL}
\label{algo:MaxKL}
\Require  A value function $V$, a probability vector $p$, a constant~$\epsilon$
\Ensure A probability vector $q$ that maximizes~\eqref{eq:max_kl}
\State Let $Z=\{i: p_i=0\}$ and $\bar{Z}=\{i: p_i>0\}$.\\ Let $I^* = Z\cap\argmax_{i}V_i$
\If{$I^*\neq\emptyset$ and there exists $i\in I^*$ such that $f(V_i)<\epsilon$}
\State Let $\nu=V_i$ and $r=1-\exp(f(\nu)-\epsilon)$.
\State For all $i\in I^*$, assign values of $q_i$ such that $$\sum_{i\in I^*}q_i=r\eqsp.$$
\State For all $i\in Z/I^*$, let $q_i=0$.
\Else
\State For all $i\in Z$, let $q_i=0$. Let $r=0$.
\State Find $\nu$ such that $f(\nu) = \epsilon$ using Newton's method.
\EndIf
\State For all $i\in\bar{Z}$, let  $q_i = \frac{(1-r)\tilde{q}_i}{\sum_{i\in\bar{Z}}\tilde{q}_i}$ where $\tilde{q}_i = \frac{p_i}{\nu-V_i}$.
\end{algorithmic}
\end{algorithm}


In practice, $f$ being a convex positive decreasing function (see Appendix~\ref{ap:proof_pratic}), Newton's method can be applied to find $\nu$ such that $f(\nu) = \epsilon$ (in Step 9 of the algorithm), so that numerically solving \eqref{eq:max_kl} is a matter of a few iterations.
Appendix~\ref{ap:proof_pratic} contains a discussion of the initialization of Newton's algorithm based on asymptotic arguments.


\section{Regret bounds}\label{sec:regret}
\subsection{Theoretical results}

To analyze the performance of KL-UCRL, we compare the rewards accumulated by the algorithm to the rewards that would be obtained, on average, by an agent playing an optimal policy. The \textit{regret} of the algorithm after $T$ steps is defined as in \cite{Jaksch:al:10}:
$$\reg_T = \sum_{t=1}^T \left(\rho^*(\M)-R_t\right)\eqsp.$$

We adapt the regret bound analysis of the UCRL2 algorithm to the use of KL-neighborhoods, and obtain similar theorems.
Let $$D(\M) = \max_{x,x'}\min_\pi\E_{\M, \pi}(\tau(x,x'))\eqsp,$$ where $\tau(x,x')$ is the hitting time of $x'$, starting from state $x$. The $D(\M)$ constant will appear in the regret bounds. For all communicating MDPs $\M$, $D(\M)$ is finite. Theorem~\ref{theo:regret} establishes an upper bound on the regret of the KL-UCRL algorithm with $C_P$ and $C_R$ defined as
$$C_P=|\Xset|\left(B+\log\left(B+\frac{1}{\log(T)}\right)\left[1 + \frac{1}{B+\frac{1}{\log(T)}}\right]\right)$$
 where
$B=\log\left(\frac{2e|\Xset|^2|\Aset|\log(T)}{\delta}\right)$ 
and 
$$C_R= \sqrt{\frac{\log\left(4|\Xset||\Aset|\log(T)/\delta\right)}{1.99}}\eqsp.$$
\begin{theorem}\label{theo:regret}
With probability $1-\delta$, it holds that for $T>5$, the regret of KL-UCRL is bounded by
$$\reg_T\leq CD(\M)|\Xset|\sqrt{|\Aset|T\log(\log(T)/\delta)}\eqsp,$$
for a constant $C\leq 24$ that does not dependent on the model.
\end{theorem}

It is also possible to prove a logarithmic upper bound for the expected regret. This bound, presented in Theorem~\ref{theo:regretLog}, depends on the model through the constant $\Delta(\M)$ defined as $\Delta(\M) = \rho^*(\M)-\max_{\pi, \rho^\pi(\M)<\rho^*(\M)} \rho^\pi(\M)$. $\Delta(\M)$ quantifies the margin between optimal and suboptimal policies.
\begin{theorem}\label{theo:regretLog}
For $T>5$, the expected regret of KL-UCRL is bounded by
$$\E(\reg_T)\leq CD^2(\M)\frac{|\Xset|^2|\Aset|\log(T)}{\Delta(\M)}+C(\M)\eqsp,$$
where $C\leq 400$ is a constant independent of the model, and $C(\M)$ is a constant which depends on the model (see \cite{Jaksch:al:10}).
\end{theorem}

\subsection{Elements of proof} 
The proof of Theorem~\ref{theo:regret} is inspired from \cite{Jaksch:al:10,Bartlett:Tewari:09}. Due to the lack of space, we only provide the main steps of the proof. 
First, the following proposition enables us to ensure that, with high probability, the true model $\M=(\Xset,\Aset,\Ptrans,\mR)$ belongs to the set of models $\Me_t$ at each time step.
\begin{proposition}\label{prop:model}
For every horizon $T\geq 1$ and for $\delta>0$,
$\P\left(\forall t\leq T\eqsp,\eqsp\M\in\Me_t\right)\geq 1-2\delta$.
\end{proposition}
The proof relies on the two following concentration inequalities due to \cite{Garivier:Leonardi:10,Garivier:Moulines:08}: for all $x\in\Xset$, $a\in\Aset$, any $C_P>0$, and $C_R>0$, it holds that
\begin{align}
&\P\left(\forall t\leq T,\eqsp KL(\hat\Ptrans_t(.;x,a),\Ptrans(.;x,a))>\frac{C_P}{N_t(x,a)}\right)\nonumber\\
&\quad\quad\leq 2e(C_P\log(T)+|\Xset|)e^{-\frac{C_P}{|\Xset|}}=1-\frac{\delta}{|\Xset||\Aset|}\label{eq:conc_ineqP}\\
&\P\left(\forall t\leq T,\quad|\hat{r}_t(x,a)-\mR(x,a)| \leq \frac{C_R}{\sqrt{N_t(x,a)}}\right)\nonumber\\
&\quad\quad\leq 4\log(T)e^{-1.99C_R}=1-\frac{\delta}{|\Xset||\Aset|}\eqsp.\nonumber
\end{align}
Then, summing over all state-action pairs, Proposition~\ref{prop:model} follows.

Using Hoeffding's inequality, with high probability, the regret at time $T$ can be written as the sum of a regret in each of the $m(T)$ episodes plus an additional term $C_e(T,\delta))=\sqrt{T\log(1/\delta)/2}$:
\begin{align*}
\reg_T&\leq \sum_{(x,a)}N_T(x,a)(\rho^*(\M)-\mR(x,a))+C_e(T,\delta)\\
&\leq\sum_{k=1}^{m(T)}\sum_{(x,a)}n_k(x,a)(\rho^*(\M)-\mR(x,a))+C_e(T,\delta)
\end{align*}
Let $P_k$ and $\pi_k$ denote, respectively, the transition probability matrix of the optimistic model and the optimal policy in the $k$-th episode ($1\leq k\leq m(T)$). It is easy to show that (see~\cite{Jaksch:al:10} for details), with probability $1-\delta$,
\begin{align*}
\reg_T\leq& \sum_{k=1}^{m(T)}\sum_{x\in\Xset}n_k(x,\pi_k(x))\bigl[\bigr.\\
&\quad+(\mR_k(x,\pi_k(x))-\mR(x,\pi_k(x)))\\
&\quad\left(\Ptrans_k(.;x,\pi_k(x))-\Ptrans(.;x,\pi_k(x))\right)'h_k\\
&\quad+\bigl. (\Ptrans(.;x,\pi_k(x))-\e_x)'h_k\bigr]\\
&\quad+C_e(T,\delta)\eqsp,
\end{align*}
where $h_k$ is a bias vector, $\e_x(y)=1$ if $x=y$ and $\e_x(y)=0$ otherwise. 
We now bound each of the three terms in the previous summation. 
Denote by $n_k^{\pi_k}$ the row vector such that $n_k^{\pi_k}(x)=n_k(x,\pi_k(x))$ and by $\mR_k^{\pi_k}$ (resp. $\mR^{\pi_k}$) the column vector such that $\mR_k^{\pi_k}(x)=\mR_k(x, \pi_k(x))$ (resp. $\mR^{\pi_k}(x)=\mR(x, \pi_k(x))$). Similarly $P_k^{\pi_k}$ (resp. $P^{\pi_k}$) is the transition matrix if the policy $\pi_k$ is followed under the optimistic model $\M_k$ (resp. the true model $\M$). 
If the true model $\M\in\Me_{t_k}$, we have for all $x\in\Xset$, for all $a\in\Aset$,
\begin{equation}
 n_k^{\pi_k}(\mR_k^{\pi_k}-\mR^{\pi_k})\leq 2\sum_{(x,a)}n_k(x,a)\sqrt{\frac{C_R}{N_t(x,a)}}\label{eq:borne1}
\end{equation}
Using Pinsker's inequality, and the fact that $\norm{h_k}{\infty}\leq D$ \cite{Jaksch:al:10},
\begin{align}
&n_k^{\pi_k}(\Ptrans_k^{\pi_k}-\Ptrans^{\pi_k})h_k\nonumber\\
&\quad\leq\sum_{(x,a)}n_k(x,a) \norm{\Ptrans_k(.;x,a)-\Ptrans(.;x,a)}{1}\norm{h_k}{\infty}\nonumber\\
 & \quad\leq2D\sqrt{2}\sum_{(x,a)}n_k(x,a)\sqrt{\frac{C_P}{N_{t_k}(x,a)}}\eqsp.\label{eq:borne2}
\end{align}
The third term $n_k^{\pi_k}(\Ptrans^{\pi_k}-I)h_k$ may be written as follows:
\begin{align*}
 n_k^{\pi_k}(\Ptrans^{\pi_k}-I)h_{k}&\leq\sum_{t=t_k}^{t_{k+1}-1}(\Ptrans(.; X_t,A_t)-\e_{X_{t+1}})h_{k} \\
&\quad+h_{k}(X_{t+1})-h_{k}(X_t)\eqsp,
\end{align*}
where $\e_{x}$ is the all 0's vector with a $1$ only on the $x$-th component. For all $t\in[t_k,t_{k+1}-1]$, note that $\xi_t=(\Ptrans(.; X_t,A_t)-\e_{X_{t+1}})h_{k}$ is a martingale difference upper-bounded by $D$. Applying the Azuma-Hoeffding inequality, we obtain that
\begin{align}
 \sum_{k=1}^{m(T)}n_k^{\pi_k}(\Ptrans^{\pi_k}-I)h_k& = \sum_{t=1}^T \xi_t + m(T)D \nonumber\\
& \leq D\sqrt{\frac{T\log(1/\delta)}{2}}+ m(T)D\label{eq:borne3}
\end{align}
with probability $1-\delta$. 
In addition, Auer and al \cite{Auer:al:09} proved that $$\sum_{k=1}^{m(T)}\sum_{x,a}\frac{n_k(x,a)}{\sqrt{N_{t_k}(x,a)}}\leq (\sqrt{2}+1)\sqrt{|\Xset||\Aset|T}$$ and $$m(T)\leq |\Xset||\Aset|\log_2\left(\frac{8T}{|\Xset||\Aset|}\right)\eqsp.$$
Combining all the terms completes the proof of Theorem~\ref{theo:regret}. The proof of Theorem~\ref{theo:regretLog} follows from Theorem~\ref{theo:regret} using the same arguments as in the proof of Theorem~4 in~\cite{Jaksch:al:10}.

\section{Numerical experiments}\label{sec:simu}
To compare the behavior of algorithms KL-UCRL and UCRL2, we consider the benchmark environments \textit{RiverSwim} and \textit{SixArms} proposed by \cite{Strehl:Littman:08} as well as a collection of randomly generated sparse environments. 
\begin{figure}[hbtp]
  \centering
  \includegraphics[width=0.9\textwidth]{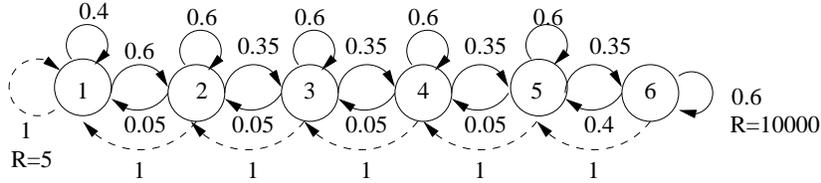}
  \caption{\textit{RiverSwim} Transition Model: the continuous (resp. dotted) arrows represent the transitions if action $1$ (resp. $2$) has been chosen.}
  \label{fig:riverSwim}
\end{figure}
The \textit{RiverSwim} environment consists of six states. The agent starts from the left side of the row and, in each state, can either swim left or right. Swimming to the right (against the current of the river) is successful with probability $0.35$; it leaves the agent in the same state with a high probability equal to $0.6$, and leads him to the left with probability $0.05$ (see Figure~\ref{fig:riverSwim}). On the contrary, swimming to the left (with the current) is always successful. The agent receives a small reward when he reaches the leftmost state, and a much larger reward when reaching the rightmost state -- the other states offer no reward. This MDP requires efficient exploration procedures, since the agent, having no prior idea of the rewards, has to reach the right side to discover which is the most valuable state-action pair.
\begin{figure}[hbtp]
  \centering
  \input{sixArms.pstex_t}
  \caption{\textit{SixArms} Transition Model}
  \label{fig:sixArms}
\end{figure}
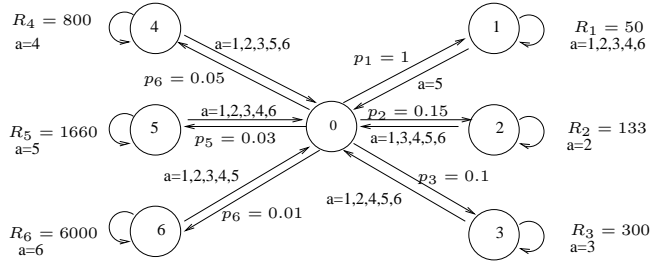
The \textit{SixArms} environment consists of seven states, one of which (state $0$) is the initial state. From the initial state, the agent may choose one among six actions: the action $a\in\{1,\dots,6\}$ leads to the state $x=a$ with probability $p_a$ (see Figure~\ref{fig:sixArms}) and let the agent in the initial state with probability $1-p_a$. 
From all the other states, some actions deterministically lead the agent to the initial state while the others leave it in the current state. Staying in a state $x\in\{1,\dots,6\}$, the agent receives a reward equal to $R_x$ (see Figure~\ref{fig:sixArms}), otherwise, no reward is received.

We compare the performance of the KL-UCRL algorithm to UCRL2 using $20$ Monte-Carlo replications. For both algorithms, the constants $C_P$ and $C_R$ are settled to ensure that the upper bounds of the regret of Theorem~\ref{theo:regret} and Theorem~2 in~\cite{Jaksch:al:10} hold with probability $0.95$. 
In the \textit{SixArms} environment, the received rewards being deterministic, we slightly modify both algorithms so that the agent knows them beforehand. 
We observe in Figure~\ref{fig:regret} and~\ref{fig:regret:6arms} that the KL-UCRL algorithm accomplishes a smaller average regret than the UCRL2 algorithm in those benchmark environments. In both environments, it is crucial for the agent to learn that there is no action leading from some states to the most promising one: for example, in the \textit{RiverSwim} environment, between one of the first four states and the sixth state.

In addition to those benchmark environments, a generator of sparse environments has been used to create $10$-states and $5$-actions environments with random rewards in $[0,1]$. In these random environments, each state is connected with, on average, five other states (with transition probabilities drawn from a Dirichlet distribution). We reproduced the same experiments as in the previous environments and display the average regret in Figure~\ref{fig:regret:random}.
\begin{figure}[hbtp]
  \centering
\includegraphics[width=0.8\textwidth]{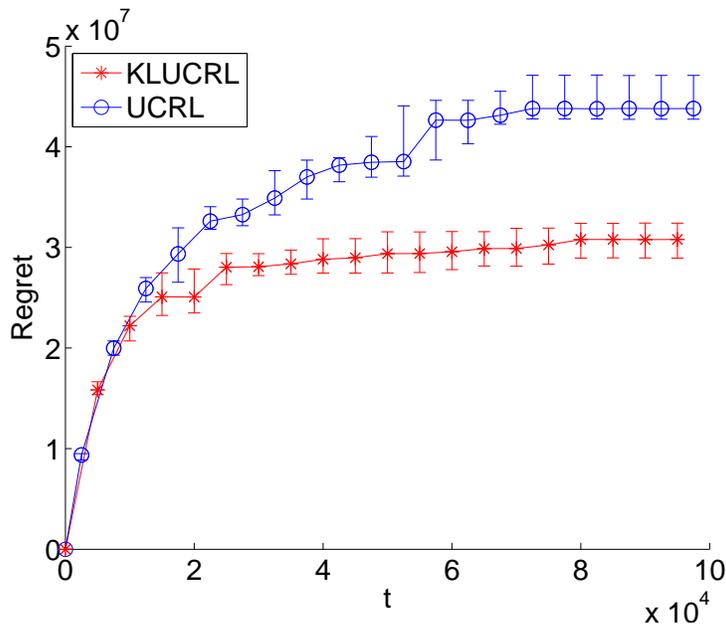}
  \caption{Comparison of the regret of the UCRL2 and KL-UCRL algorithms in the RiverSwim environment.}
  \label{fig:regret}
\end{figure}
\begin{figure}[hbtp]
  \centering
\includegraphics[width=0.8\textwidth]{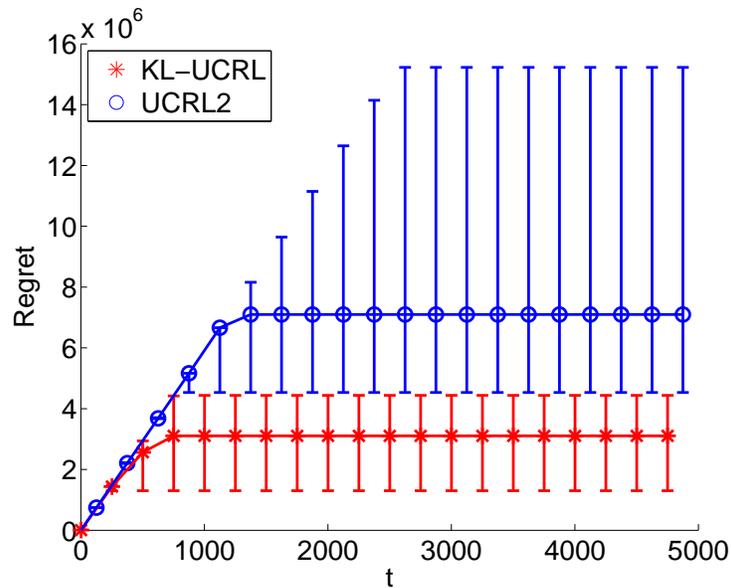}
  \caption{Comparison of the regret of the UCRL2 and KL-UCRL algorithms in the SixArms environment.}
  \label{fig:regret:6arms}
\end{figure}

\begin{figure}[hbtp]
  \centering
\includegraphics[width=0.8\textwidth]{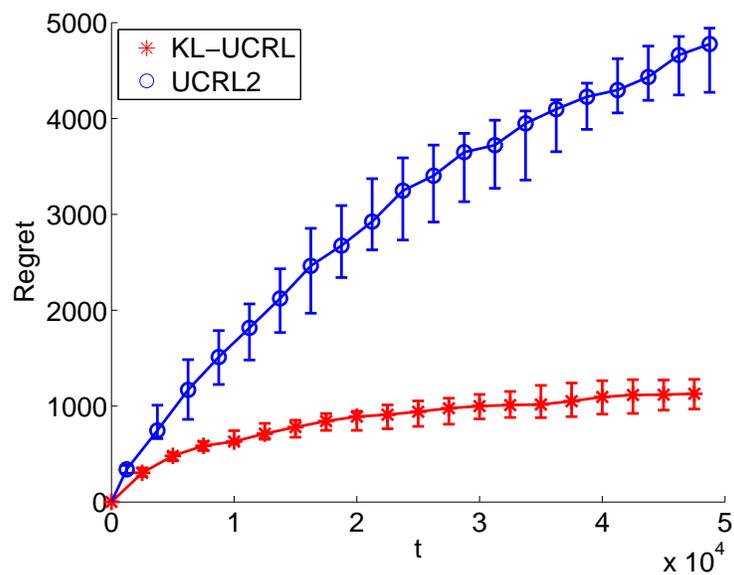}
  \caption{Comparison of the regret of the UCRL2 and KL-UCRL algorithms in randomly generated sparse environments.}
  \label{fig:regret:random}
\end{figure}

\section{Discussion}\label{sec:KLversusL1}
 In this section, we expose the advantages of using a confidence ball based on the Kullback-Leibler divergence rather than an $L^1$-ball, as proposed for instance in \cite{Jaksch:al:10,Tewari:Bartlett:08}, in the computation of the optimistic policy. This discussion aims at explaining and interpreting the difference of performance that can be observed in simulations.
In KL-UCRL, optimism reduces to maximizing the linear function $V'q$ over a KL-ball (see~\eqref{eq:max_kl}), whereas the other algorithms make use of an $L^1$-ball:
\begin{equation}
\max_{q\in\S^{|\Xset|}}V'q\quad\text{s.t.}\quad \norm{p-q}{1}\leq\epsilon'\eqsp.\label{eq:max_1} 
\end{equation}

\subsection*{Continuity}
Consider an estimated transition probability vector $p$, and denote by $q^{KL}$ (resp. $q^1$) the probability vector which maximizes Equation~\eqref{eq:max_kl} (resp. Equation~\eqref{eq:max_1}).
It is easily seen that $q^{KL}$ and $q^{1}$ lie respectively on the border of the convex set $\{q\in\S^{|X|} : KL(p,q)\leq\epsilon\}$ and at one of the vertices of the polytope $\{q\in\S^{|X|} : \norm{p-q}{1}\leq\epsilon'\}$. A first noteworthy difference between those neighborhoods is that, due to the smoothness of the KL-neighborhood, $q^{KL}$ is continuous with respect to the vector $V$, which is not the case for $q_{1}$.

To illustrate this, Figure~\ref{fig:simplexe_var} displays $L^1$- and KL-balls around $3$-dimensional probability vectors. The set of $3$-dimensional probability vectors is represented by a triangle whose vertices are the vectors $(1,0,0)'$, $(0,1,0)'$ and $(0,0,1)'$, the probability vector $p$ by a white star, and the vectors $q^{KL}$ and $q^1$ by a white point. The arrow represents the direction of $V$'s projection on the simplex and indicates the gradient of the linear function to maximize. The maximizer $q^1$ can vary significantly for small changes of the value function, while $q^{KL}$ varies continuously.
\begin{figure}[hbt] \centering
 \includegraphics[width=0.44\textwidth]{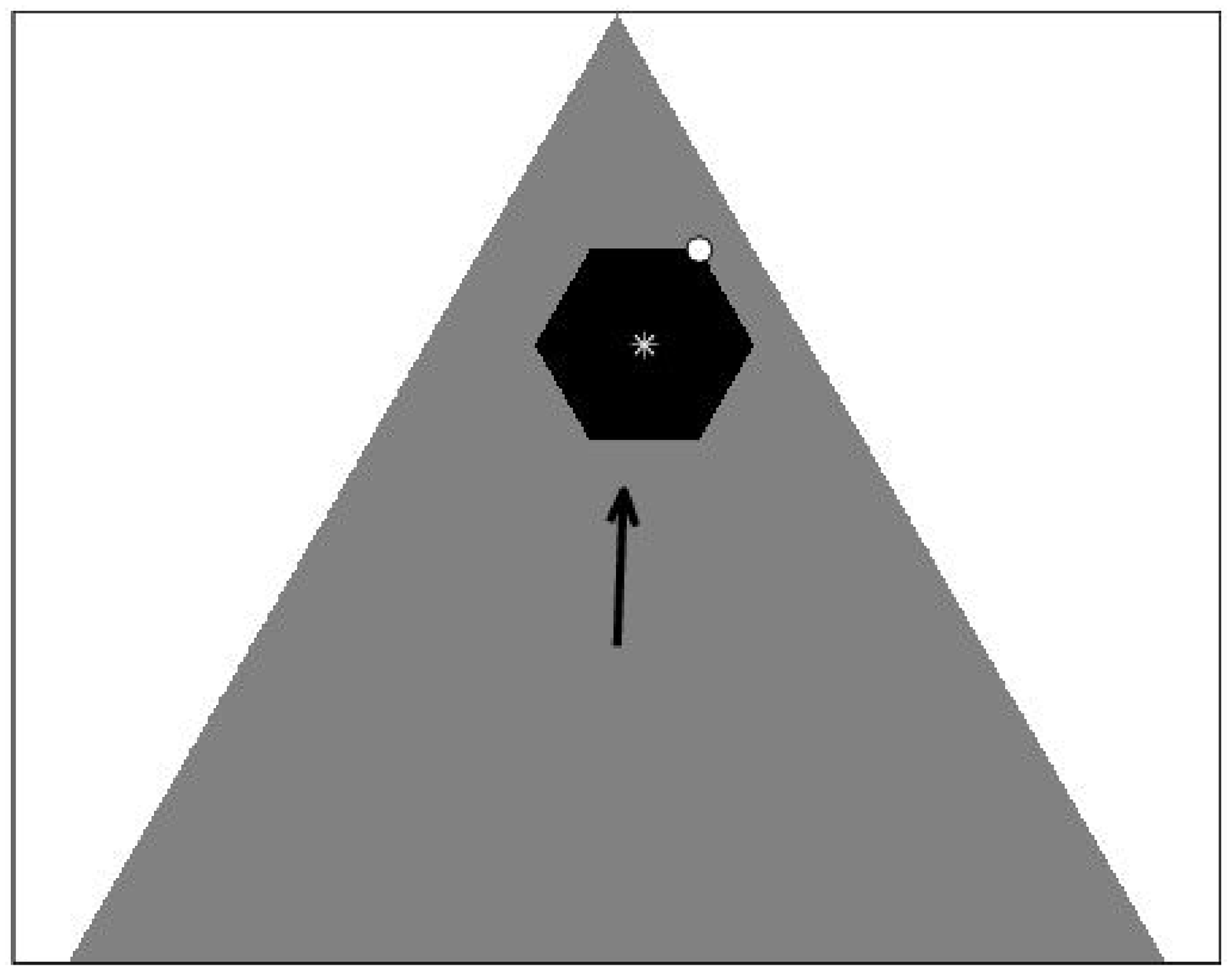}
\includegraphics[width=0.44\textwidth]{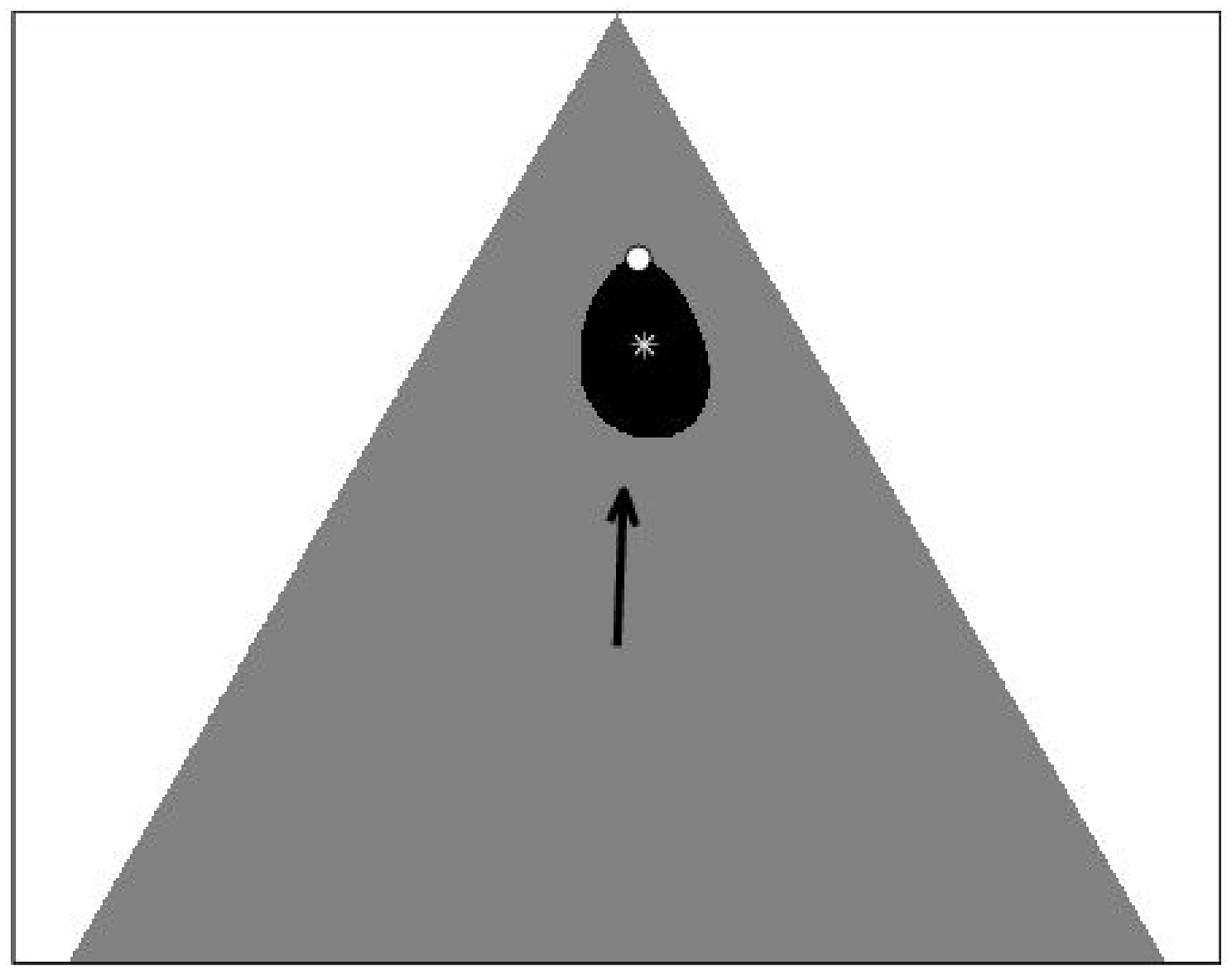} \\
\includegraphics[width=0.44\textwidth]{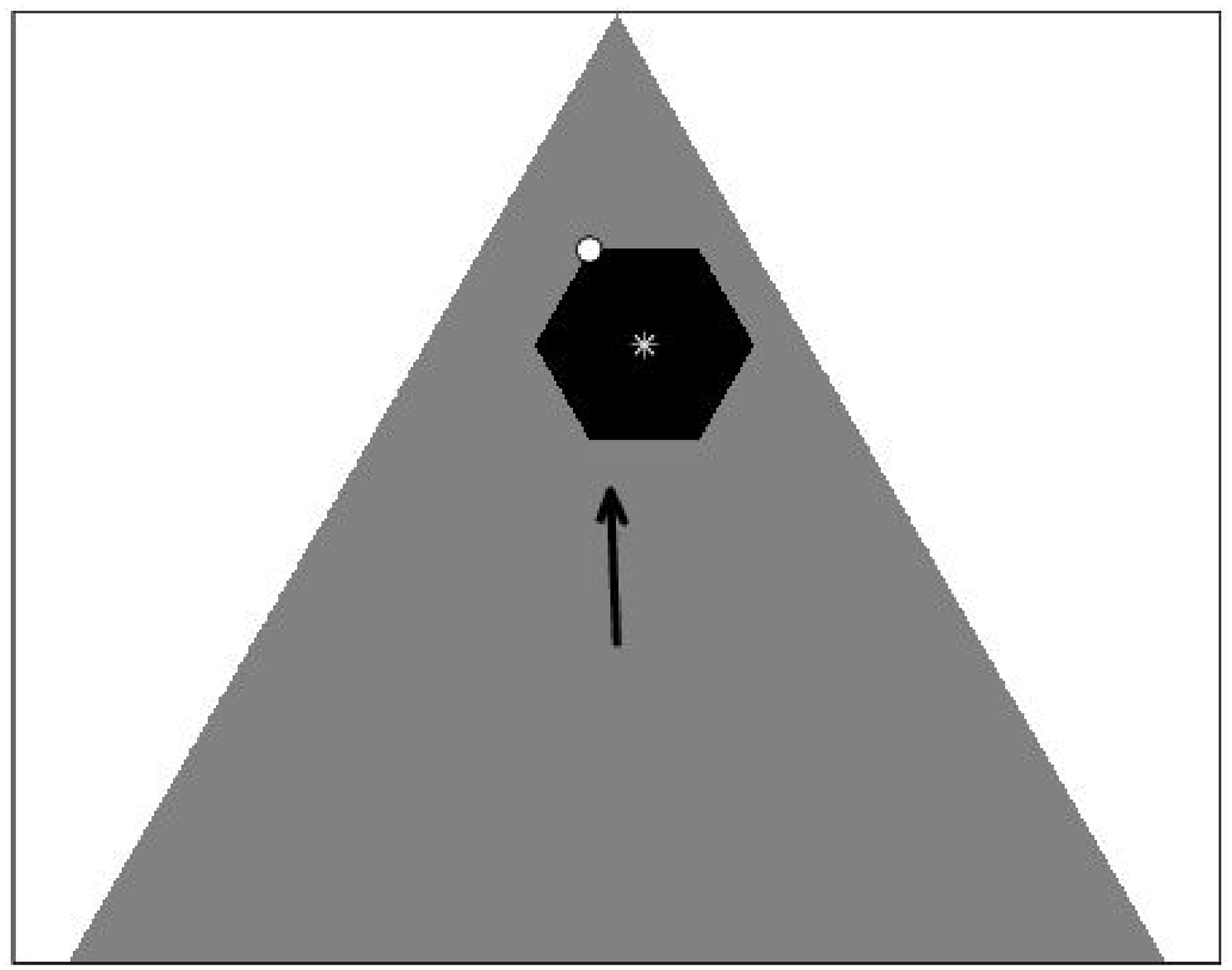}
\includegraphics[width=0.44\textwidth]{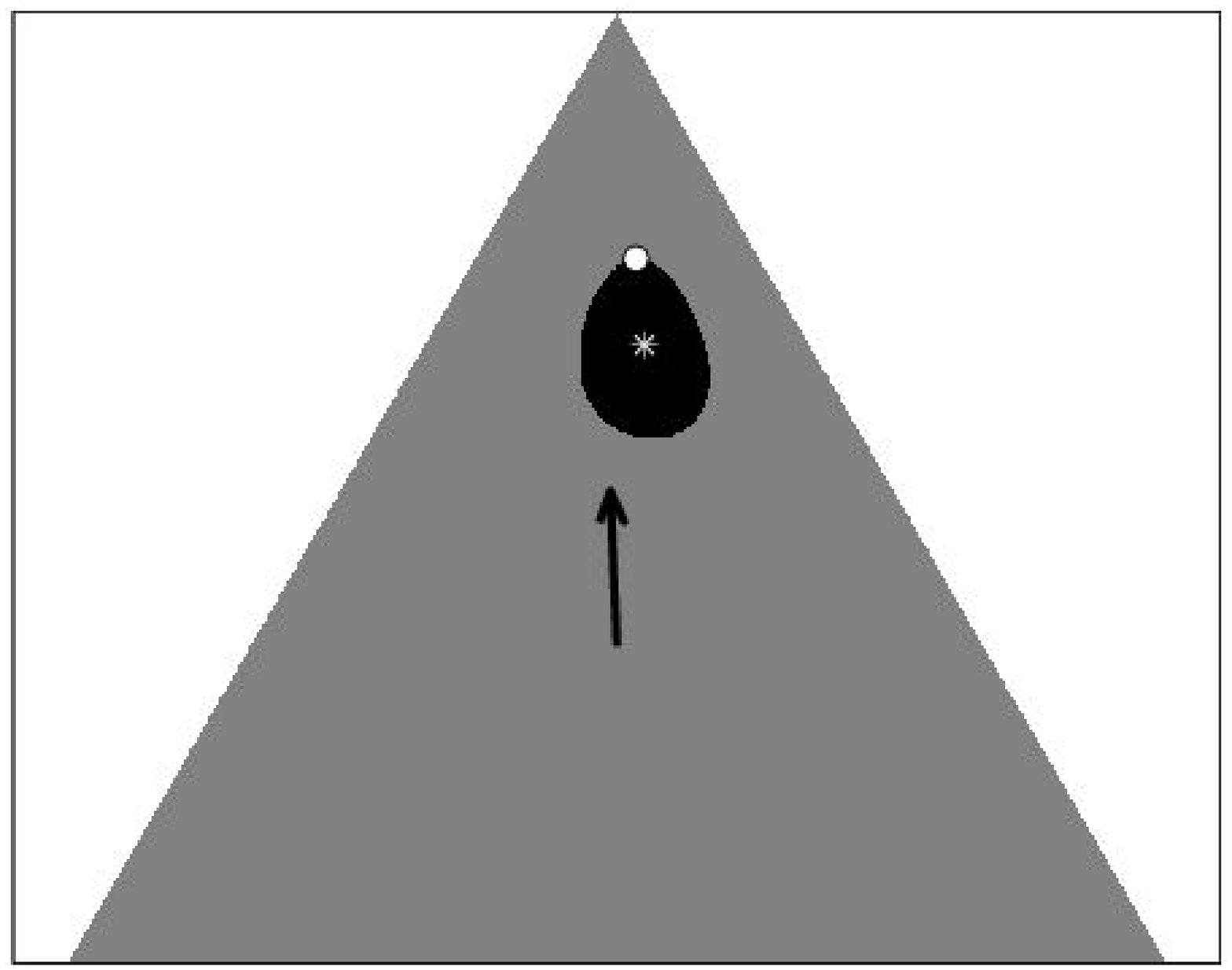}
\caption{The $L^1$-neighborhood  $\{q\in\S^3 : \norm{p-q}{1}\leq 0.2\}$ (left) and KL-neighborhood  $\{q\in\S^3 : KL(p,q)\leq 0.02\}$ (right) around the probability vector $p=(0.15,0.2,0.65)'$ (white star). The white points are the maximizers of equations~\eqref{eq:max_kl} and~\eqref{eq:max_1} with $V=(0,0.05,1)'$ (up) and $V=(0,-0.05,1)'$ (down).}
\label{fig:simplexe_var}
\end{figure}
\vspace*{-0.5cm}
\subsection*{Unlikely transitions}
Denote $i_m=\argmin_j V_j$ and $i_M=\argmax_j V_j$. As underlined by~\cite{Jaksch:al:10}, $q^1_{i_m}=\max(p_{i_m}-\epsilon'/2,0)$ and $q^1_{i_M}=\min(p^1_{i_M}+\epsilon'/2,1)$. This has two consequences:
\begin{enumerate}
 \item if $p$ is such that $0<p_{i_m}<\epsilon'/2$, then the vector $q^1_{i_m}=0$; so the optimistic model may assign a probability equal to zero to a transition that has actually been observed, which makes it hardly compatible with the optimism principle. Indeed, an optimistic MDP should not forbid transitions that really exists, even if they lead to states with small values; 
 \item if $p$ is such that $p_{i_M}=0$, then $q^1_{i_M}$ never equals $0$; therefore, an optimistic algorithm that uses $L^1$-balls will always assign positive probability to transitions to $i_M$ even if this transition is impossible under the true MDP and if much evidence has been accumulated against the existence of such a transition.
Thus, the exploration bonus of the optimistic procedure is wasted, whereas it could be used more efficiently to favor some other transitions.
\end{enumerate}
This explains a large part of the experimental advantage of KL-UCRL observed in the simulations.
Indeed, $q^{KL}$ always assigns strictly positive probability to observed transitions, and eventually renounces unobserved transitions even if the target states have a potentially large value. Algorithm~\ref{algo:MaxKL} works as follows: for all $i$ such that $p_i\neq 0$, $q_i\neq 0$; for all $i$ such that $p_i=0$, $q_i=0$ except if $p_{i_M}=0$  and if $f(V_{i_M})<\epsilon$, in which case $q_{i_M}=1-\exp(f(V_{i_M})-\epsilon)$. But this is no longer the case when $\epsilon$ becomes small enough, that is, when sufficiently many observations are available. 
We illustrate those two important differences in Figure~\ref{fig:simplexe_0}, by representing the $L^1$ and KL neighborhoods together with the maximizers $q^{KL}$ and $q^{1}$, first if $p_{i_m}$ is positive by very small, and second if $p_{i_M}$ is equal to $0$. Figure~\ref{fig:evolution} also illustrates the latter case, by representing the evolution of the probability vector $q$ that maximizes both \eqref{eq:max_1} and \eqref{eq:max_kl} for an example with $p=(0.3,0.7,0)'$, $V=(1,2,3)'$ and $\epsilon$ decreasing from $1/2$ to $1/500$. 

\vspace{-0.2cm}
\begin{figure}[hbtp] \centering
 \includegraphics[width=0.44\textwidth]{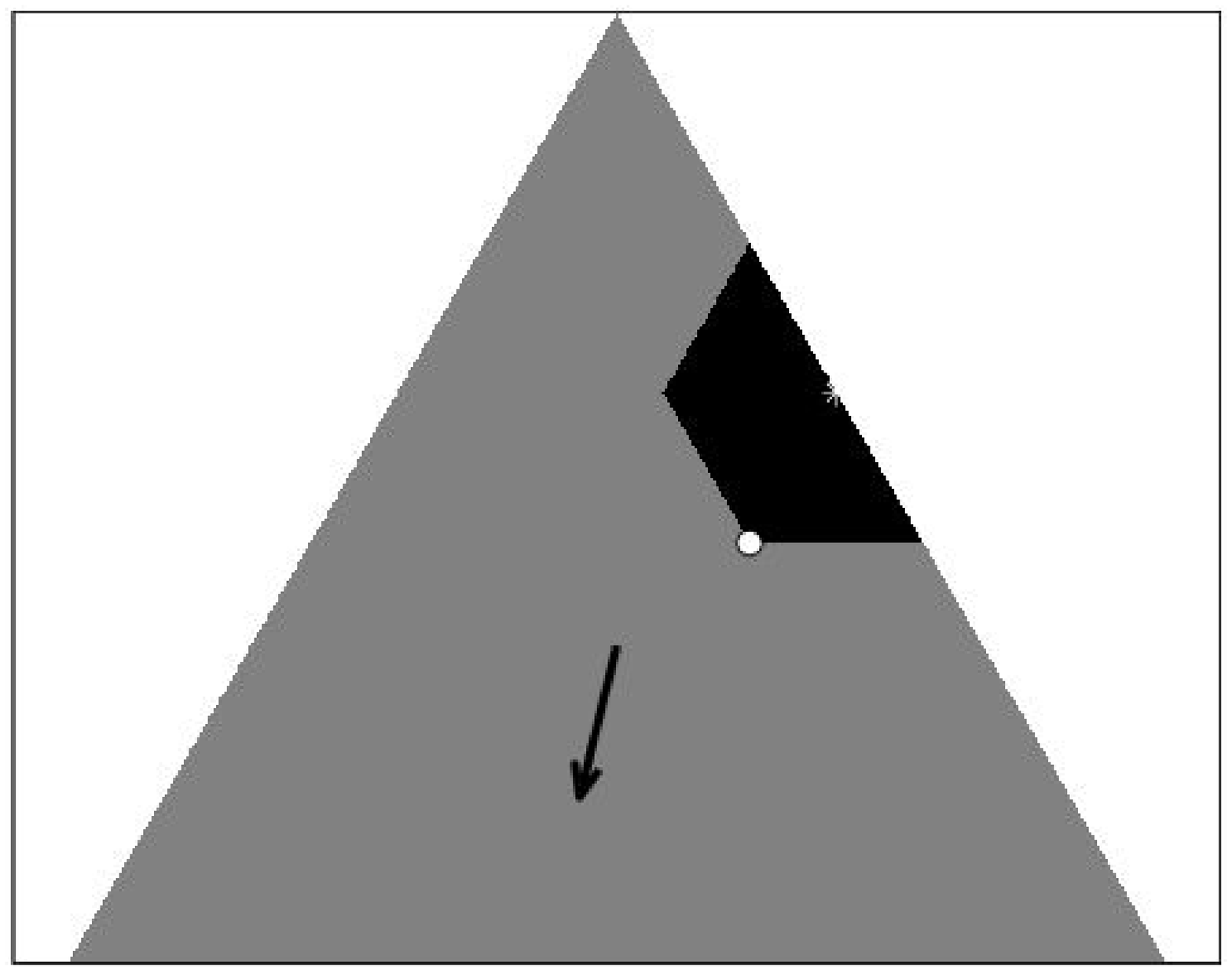}
\includegraphics[width=0.44\textwidth]{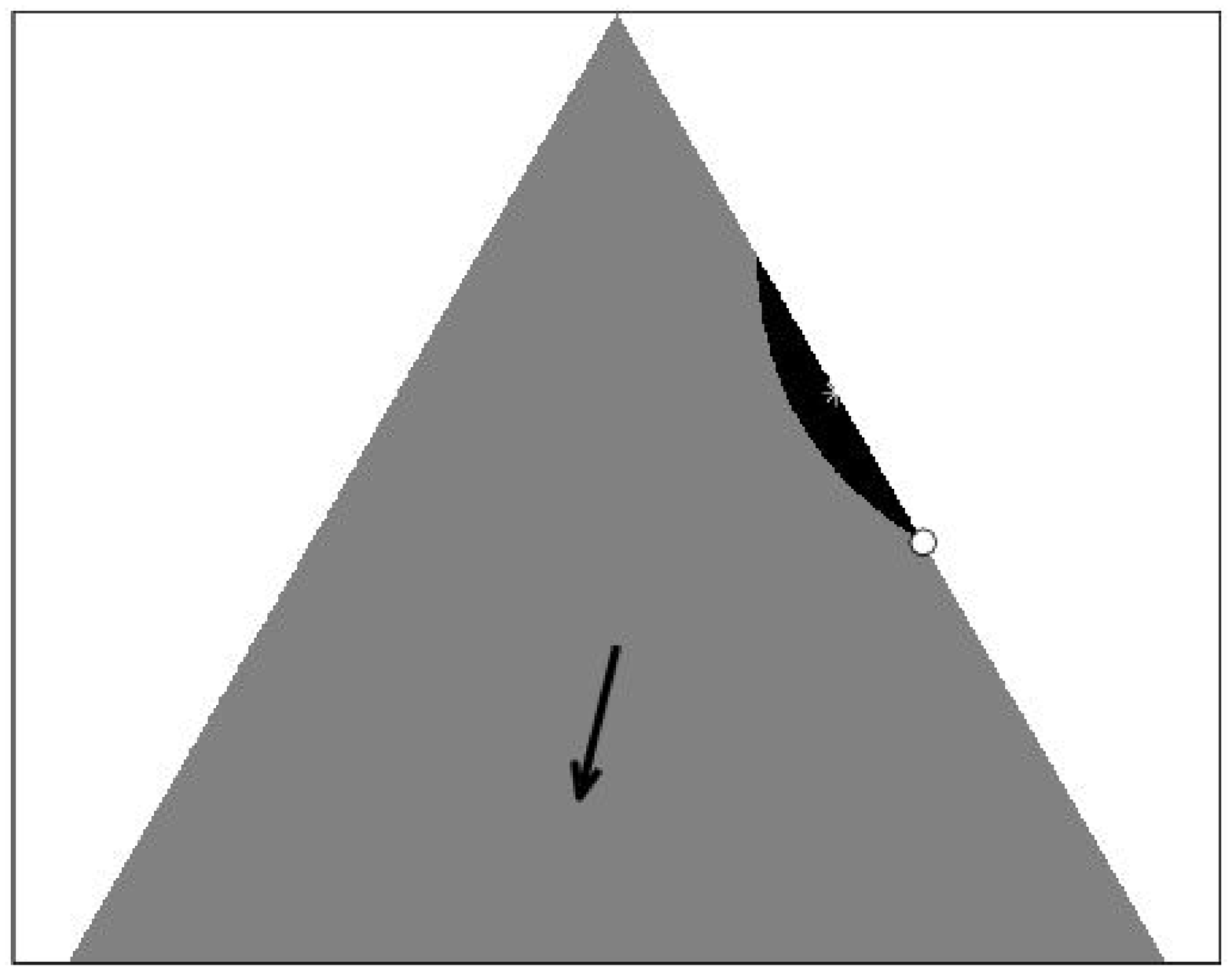} \\
 \includegraphics[width=0.44\textwidth]{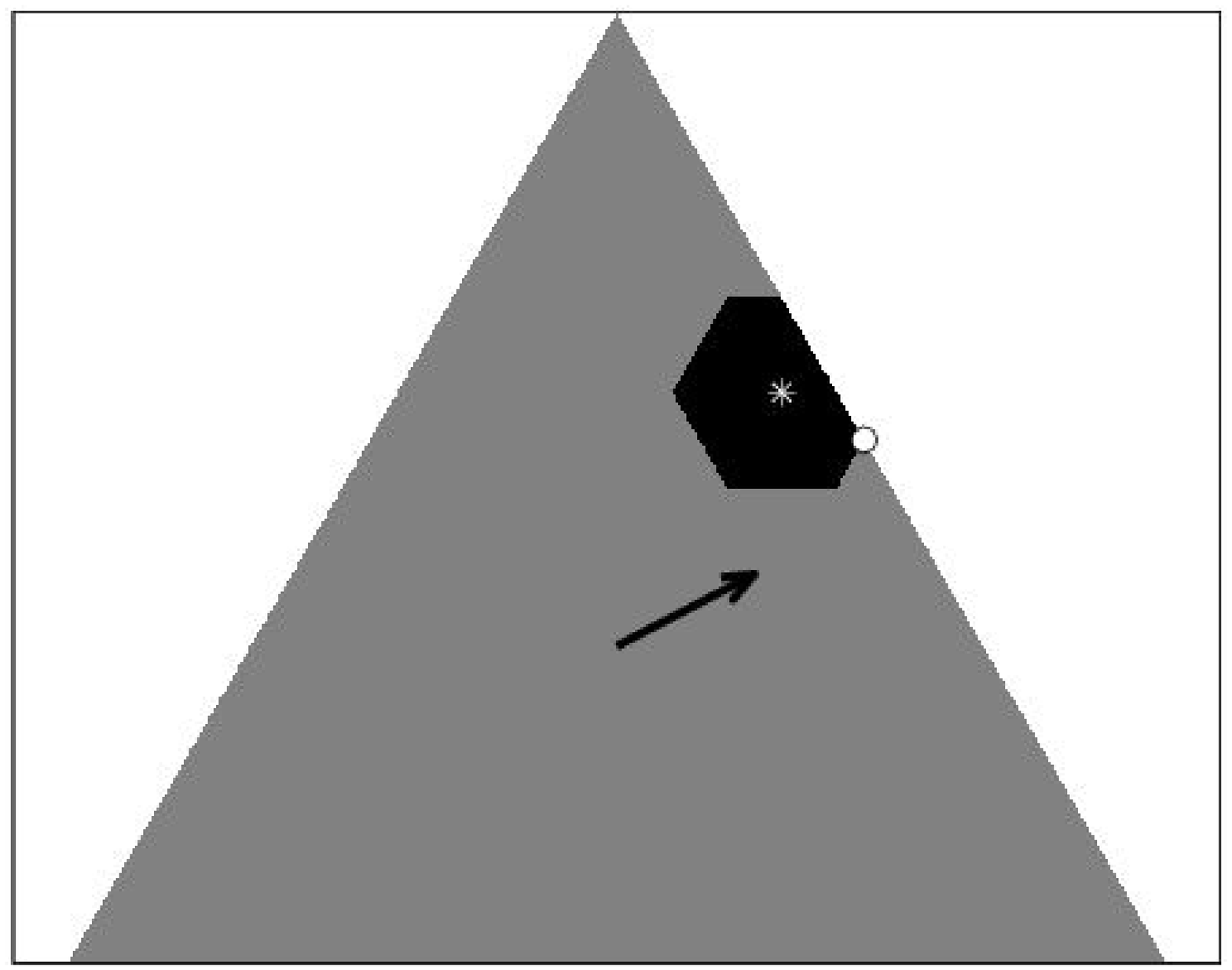}
\includegraphics[width=0.44\textwidth]{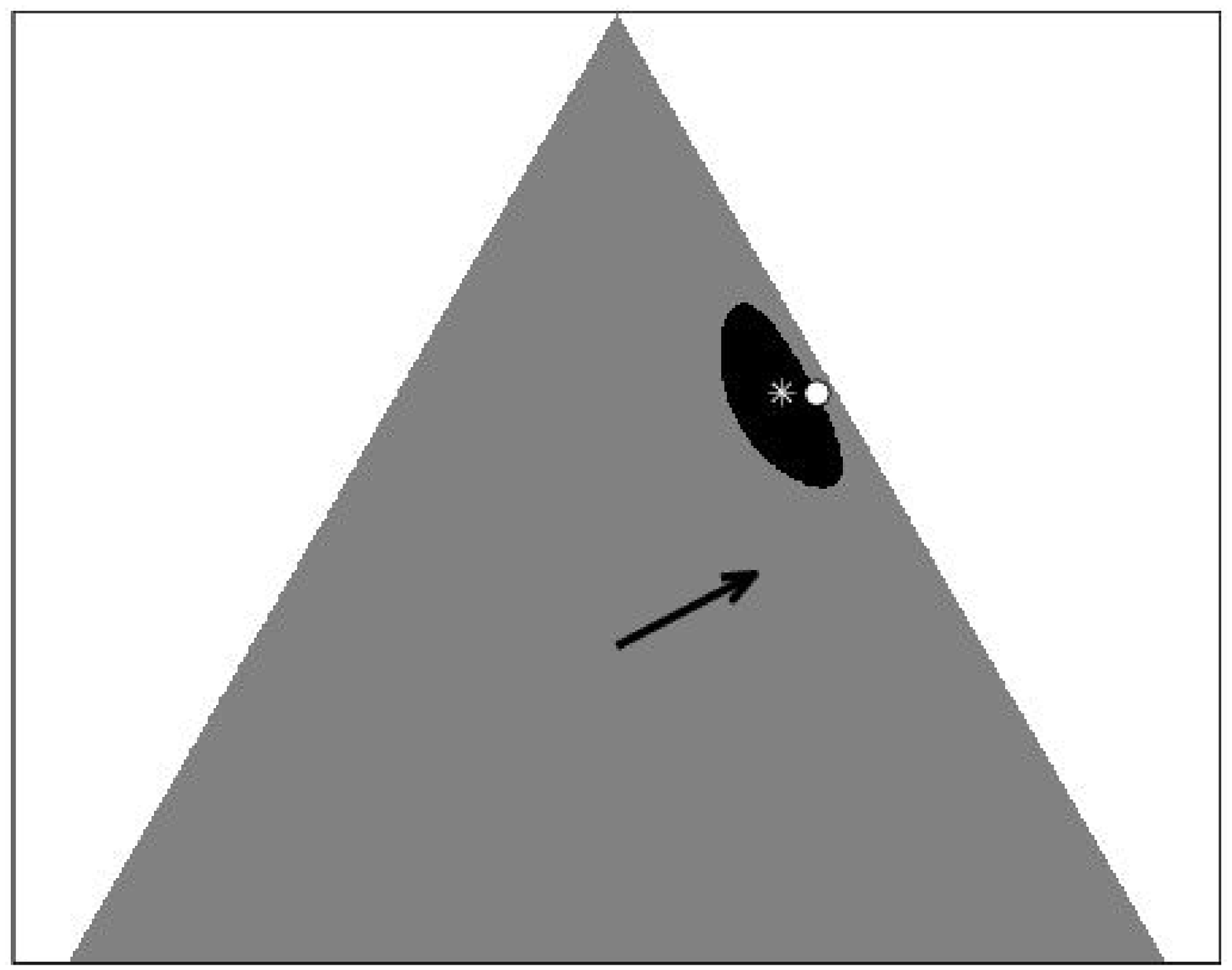}
\caption{The $L^1$ (left) and KL-neighborhoods (right) around the probability vector $p=(0,0.4,0.6)'$ (up) and $p=(0.05,0.35,0.6)'$ (down). The white point is the maximizer of the equations~\eqref{eq:max_kl} and~\eqref{eq:max_1} with $V=(-1,-2,-5)'$ (up) and $V=(-1,0.05,0)'$ (down). We took, $\epsilon=0.05$ (up), $\epsilon=0.02$(down) and $\epsilon'=\sqrt{2\epsilon}$.}
\label{fig:simplexe_0}
\end{figure}
\begin{figure}[hbtp]
  \centering
\includegraphics[width=0.9\textwidth]{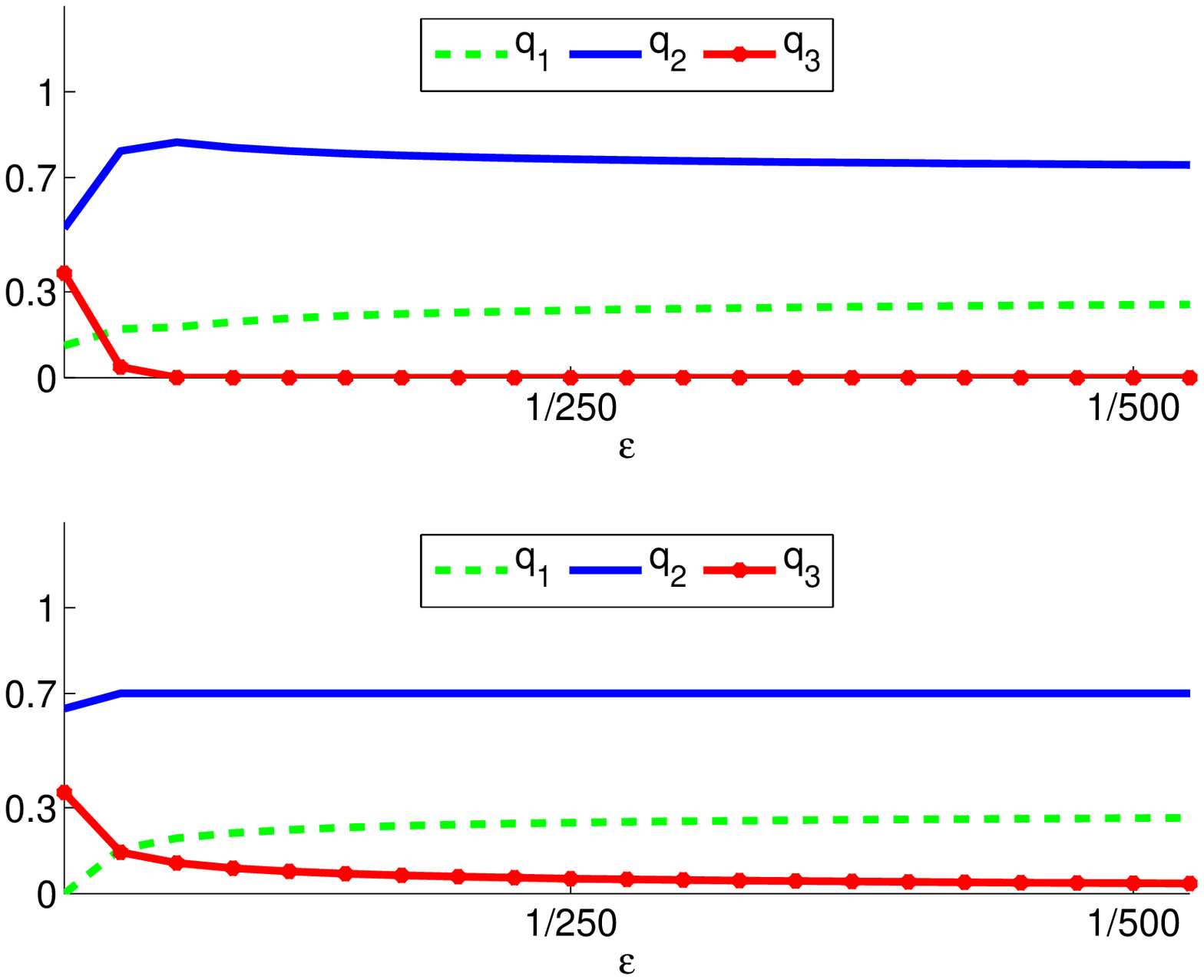}
  \caption{Evolution of the probability vector $q$ that maximizes both \eqref{eq:max_kl} (top) and \eqref{eq:max_1} (bottom) with $p=(0.3,0.7,0)'$, $V=(1,2,3)'$ and $\epsilon$ decreasing from $1/2$ to $1/500$\vspace*{-1cm}}
  \label{fig:evolution}
\end{figure}
\appendix
\section{Linear optimization over a KL-ball}\label{ap:solmax}
This section explains how to solve the optimization problem of Equation~\eqref{eq:max_kl}.
In \cite{NilimElghaoui:05}, a similar problem arises in a different context, and a somewhat different solution is proposed for the case when the $p_i$ are all positive.
As a problem of maximizing a linear function under convex constraints, it is sufficient to consider the Lagrangian function
\begin{align*}
L(q,\lambda, \nu, \mu_1,\dots,\mu_N) &= \sum_{i=1}^Nq_iV_i - \lambda\left(\sum_{i=1}^Np_i\log\frac{p_i}{q_i}-\epsilon\right)\\
&  - \nu\left(\sum_{i=1}^Nq_i-1\right)+\sum_{i=1}^N\mu_iq_i\eqsp.
\end{align*}
If $q$ is a maximizer, there exist $\lambda\in\Rset$, $\nu,\mu_i \geq 0$  $(i=1\dots N)$ such that the following conditions are simultaneously satisfied:\\
\begin{numcases}
   \strut V_i  +\lambda \frac{p_i}{q_i} -\nu+\mu_i = 0\label{eq:condLagpos}\\
    \lambda\left(\sum_{i=1}^Np_i\log\frac{p_i}{q_i}-\epsilon\right) = 0\label{eq:condKL}\\
    \nu\left(\sum_{i=1}^Nq_i-1\right) = 0\label{eq:condsum1}\\
    \mu_iq_i = 0\label{eq:condpos}
\end{numcases}
Let $Z = \{i, p_i=0\}$. Conditions~\eqref{eq:condLagpos} to \eqref{eq:condpos} imply that $\lambda\neq0$ and $\nu\neq0$.
For $i\in\bar{Z}$, Equation \eqref{eq:condLagpos} implies that $q_i = \lambda\frac{p_i}{\nu-\mu_i-V_i}$.
Since $\lambda\neq0$, $q_i>0$ and then, according to \eqref{eq:condpos}, $\mu_i=0$. Therefore,
\begin{equation}
\forall i\in\bar{Z}\eqsp,\quad q_i = \lambda\frac{p_i}{\nu-V_i}\eqsp.\label{eq:qcondLagpos}
\end{equation}
Let $r=\sum_{i\in Z}q_i$. Summing on $i\in\bar{Z}$ and using equations~\eqref{eq:qcondLagpos} and~\eqref{eq:condsum1}, we have
\begin{equation}
 \lambda\sum_{i\in\bar{Z}}\frac{p_i}{\nu-V_i}=\sum_{i\in\bar{Z}}q_i=1-r\eqsp.\label{eq:temp1}
\end{equation}
 Using \eqref{eq:qcondLagpos} and \eqref{eq:temp1}, we can write $\sum_{i\in\bar{Z}}p_i\log\frac{p_i}{q_i} = f(\nu)- \log(1-r)$ where $f$ is defined in \eqref{eq:def_f}.
Then, $q$ satisfies condition \eqref{eq:condKL} if and only if 
$f(\nu) = \epsilon +\log(1-r)\eqsp.\label{eq:condKLeq} $

Consider now the case where $i\in Z$. Let $I^* = Z\cap\argmax_{i}V_i$. 
Note that, for all $i\in Z \setminus I^*$, $q_i=0$. Indeed, otherwise, $\mu_i$ should be zero, and then $\nu=V_i$ according to~\eqref{eq:condLagpos}, which involves a possible negative denominator in~\eqref{eq:qcondLagpos}. 
According to~\eqref{eq:condpos}, for all $i\in I^*$, either $q_i=0$ or $\mu_i =0$. The second case implies that $\nu=V_i$ and $r>0$ which requires that $f(\nu)<\epsilon$ so that~\eqref{eq:condKLeq} can be satisfied with $r>0$. Therefore, 
\begin{itemize}
 \item if $f(V_i)<\epsilon$ for $i\in I^*$, then $\nu=V_i$ and the constant $r$ can be computed solving equation $f(\nu) = \epsilon-\log(1-r)$; the values of $q_i$ for $i\in I^*$ may be chosen in any way such that $\sum_{i\in I^*}q_i=r$;
 \item if for all $i\in I^*$ $f(V_i)\geq\epsilon$, then $r=0$, $q_i=0$ for all $i\in Z$ and $\nu$ is the solution of the equation $f(\nu)=\epsilon$.
\end{itemize}
 Once $\nu$ and $r$ have been determined, the other components of $q$ can be computed according to~\eqref{eq:qcondLagpos}: we have that for $i\in\bar{Z}$, $q_i = \frac{(1-r)\tilde{q}_i}{\sum_{i\in\bar{Z}}\tilde{q}_i}$ where $\tilde{q}_i = \frac{p_i}{\nu-V_i}$.

\section{Properties of the $f$ function}\label{ap:proof_pratic}
In this section, a few properties of function $f$ defined in Equation~\eqref{eq:def_f} are stated, as this function plays a key role in the maximizing procedure of Section~\ref{sec:maxKL}.
\begin{proposition}\label{theo:convexity_andco}
 $f$ is a convex, decreasing mapping from $]\max_{i\in\bar{Z}} V_i;\infty[$ onto $]0;\infty[$.
\end{proposition}
\begin{proof}
Using Jensen's inequality, it is easily shown that the $f$ function decreases from $+\infty$ to $0$.
The second derivative of $f$ with respect to $\nu$ is equal to
$$-\sum_i\frac{ p_i}{(\nu-V_i)^2}+\frac{2\sum_i\frac{p_i}{(\nu-V_i)^3}\sum_i\frac{p_i}{\nu-V_i}-\left(\sum_i\frac{p_i}{(\nu-V_i)^2}\right)^2}{\left(\sum_i\frac{p_i}{\nu-V_i}\right)^2}.$$
If $Z$ denotes a positive random value such that $\P\left(Z=\frac{1}{\nu-V_i}\right)=p_i$, then 
$$f''(\nu) =\frac{2\E(Z^3)\E(Z)-\E(Z^2)\E(Z)^2-\E(Z^2)^2}{\E(Z)^2}\eqsp.$$
Using Cauchy-Schwartz inequality, we have $\E(Z^2)^2 = \E(Z^{3/2}Z^{1/2})^2\leq \E(Z^3)\E(Z)$. In addition $\E(Z^2)^2\geq \E(Z^2)\E(Z)^2\eqsp.$ These two inequalities show that $f''(\nu)\geq 0$. 
\end{proof}
As mentioned in Section~\ref{sec:maxKL}, Newton's method can be applied to solve the equation $f(\nu)=\epsilon$ for a fixed value of $\epsilon$. When $\epsilon$ is close to $0$, the solution of this equation is quite large and an appropriate initialization accelerates convergence. 
Using a second-order Taylor's-series approximation of the function $f$, it can be seen that, for $\nu$ near $\infty$, $f(\nu)=\frac{\sigma_{p,V}}{2\nu^2}+o(\frac{1}{\nu^2})$, where $\sigma_{p,V}=\sum_{i}p_iV_i^2-(\sum_ip_iV_i)^2$. 
The Newton iterations can thus be initialized by taking $\nu_0=\sqrt{\sigma_{p,V}/(2\epsilon)}$.
\bibliographystyle{abbrv}
\bibliography{biblio}

\end{document}

%% file: sixArms.pstex_t
\begin{picture}(0,0)%
\includegraphics{sixArms.pstex}%
\end{picture}%
\setlength{\unitlength}{1865sp}%
\begingroup\makeatletter\ifx\SetFigFontNFSS\undefined%
\gdef\SetFigFontNFSS#1#2#3#4#5{%
  \reset@font\fontsize{#1}{#2pt}%
  \fontfamily{#3}\fontseries{#4}\fontshape{#5}%
  \selectfont}%
\fi\endgroup%
\begin{picture}(8607,3431)(-14,-2834)
\put( 46,299){\makebox(0,0)[lb]{\smash{{\SetFigFontNFSS{6}{7.2}{\rmdefault}{\mddefault}{\updefault}{\color[rgb]{0,0,0}$R_4=800$}%
}}}}
\put(  1,-1186){\makebox(0,0)[lb]{\smash{{\SetFigFontNFSS{6}{7.2}{\rmdefault}{\mddefault}{\updefault}{\color[rgb]{0,0,0}$R_5=1660$}%
}}}}
\put(  1,-2536){\makebox(0,0)[lb]{\smash{{\SetFigFontNFSS{6}{7.2}{\rmdefault}{\mddefault}{\updefault}{\color[rgb]{0,0,0}$R_6=6000$}%
}}}}
\put(7516,209){\makebox(0,0)[lb]{\smash{{\SetFigFontNFSS{6}{7.2}{\rmdefault}{\mddefault}{\updefault}{\color[rgb]{0,0,0}$R_1=50$}%
}}}}
\put(7426,-1141){\makebox(0,0)[lb]{\smash{{\SetFigFontNFSS{6}{7.2}{\rmdefault}{\mddefault}{\updefault}{\color[rgb]{0,0,0}$R_2=133$}%
}}}}
\put(7471,-2491){\makebox(0,0)[lb]{\smash{{\SetFigFontNFSS{6}{7.2}{\rmdefault}{\mddefault}{\updefault}{\color[rgb]{0,0,0}$R_3=300$}%
}}}}
\put(1801,-466){\makebox(0,0)[lb]{\smash{{\SetFigFontNFSS{6}{7.2}{\rmdefault}{\mddefault}{\updefault}{\color[rgb]{0,0,0}$p_6=0.05$}%
}}}}
\put(4591,-196){\makebox(0,0)[lb]{\smash{{\SetFigFontNFSS{6}{7.2}{\rmdefault}{\mddefault}{\updefault}{\color[rgb]{0,0,0}$p_1=1$}%
}}}}
\put(4771,-916){\makebox(0,0)[lb]{\smash{{\SetFigFontNFSS{6}{7.2}{\rmdefault}{\mddefault}{\updefault}{\color[rgb]{0,0,0}$p_2=0.15$}%
}}}}
\put(5446,-1771){\makebox(0,0)[lb]{\smash{{\SetFigFontNFSS{6}{7.2}{\rmdefault}{\mddefault}{\updefault}{\color[rgb]{0,0,0}$p_3=0.1$}%
}}}}
\put(2836,-2266){\makebox(0,0)[lb]{\smash{{\SetFigFontNFSS{6}{7.2}{\rmdefault}{\mddefault}{\updefault}{\color[rgb]{0,0,0}$p_6=0.01$}%
}}}}
\put(2476,-1276){\makebox(0,0)[lb]{\smash{{\SetFigFontNFSS{6}{7.2}{\rmdefault}{\mddefault}{\updefault}{\color[rgb]{0,0,0}$p_5=0.03$}%
}}}}
\end{picture}%